\documentclass[twoside]{article}
\usepackage[accepted]{aistats2015}

\usepackage{graphicx}
\usepackage{psfrag}
\usepackage{epsfig,subfigure,epstopdf}
\usepackage[wide,rightcaption]{sidecap}
\usepackage{amsmath,amssymb,amsfonts,amsbsy,amsthm}
\usepackage{latexsym}
\usepackage{algorithm,algorithmic}
\usepackage{url}
\usepackage{times}
\usepackage[round]{natbib}
\usepackage{color}

\input{def.set}

\newcommand{\sfd}{\mathsf d}
\newcommand{\pdpm}{p_{\mathrm{\tiny DPM}}}
\newcommand{\ip}[1]{\langle #1 \rangle}
\usepackage[font=small]{caption}
\begin{document}

%

%

\runningtitle{BHC with Exponential Family: Small-Variance Asymptotics and Reducibility}
\runningauthor{Lee, Choi}

\twocolumn[

\aistatstitle{Bayesian Hierarchical Clustering with Exponential Family: \\
Small-Variance Asymptotics and Reducibility}

\aistatsauthor{ Juho Lee \and Seungjin Choi }

\aistatsaddress{ Department of Computer Science and Engineering \\
Pohang University of Science and Technology \\
77 Cheongam-ro, Nam-gu, Pohang 790-784, Korea \\
{\tt \{stonecold,seungjin\}@postech.ac.kr}
} ]

\begin{abstract}
Bayesian hierarchical clustering (BHC) is an agglomerative clustering method, where a probabilistic model is defined
and its marginal likelihoods are evaluated to decide which clusters to merge.
While BHC provides a few advantages over traditional distance-based agglomerative clustering algorithms,
successive evaluation of marginal likelihoods and careful hyperparameter tuning are cumbersome and limit the scalability.
In this paper we relax BHC into a non-probabilistic formulation, exploring small-variance asymptotics in
conjugate-exponential models.
We develop a novel clustering algorithm, referred to as {\em relaxed BHC} (RBHC),  
from the asymptotic limit of the BHC model that exhibits the scalability of 
distance-based agglomerative clustering algorithms as well as the flexibility of Bayesian nonparametric models.
We also investigate the reducibility of the dissimilarity measure emerged from the asymptotic limit of the BHC model,
allowing us to use scalable algorithms such as the nearest neighbor chain algorithm.
Numerical experiments on both synthetic and real-world datasets demonstrate the validity and high performance of our method.
\end{abstract}

\section{INTRODUCTION}
\label{sec:introduction}

Agglomerative hierarchical clustering, which is one of the most popular algorithms in cluster analysis,
builds a binary tree representing the cluster structure of a dataset~\citep{DudaRO2001book}. 
Given a dataset and a dissimilarity measure between clusters, 
agglomerative hierarchical clustering starts from leaf nodes corresponding to individual data points
and successively merges pairs of nodes with smallest dissimilarities to complete a binary tree.

Bayesian hierarchical clustering (BHC) \citep{HellerKA2005icml} is a probabilistic alternative of agglomerative 
hierarchical clustering. BHC defines a generative model on binary trees and compute the probability that 
nodes are merged under that generative model to evaluate the (dis)similarity between the nodes.
Since this (dis)similarity is written as a probability, one can naturally decide a level, where to stop
merging according to this probability. Hence, unlike traditional agglomerative clustering algorithms,
BHC has a flexibility to infer a proper number of clusters for given data. The source of this
flexibility is Dirichlet process mixtures (DPM) \citep{FergusonTS73as,AntoniakCE74aos} used to define the generative model of binary trees. 
BHC was shown to provide a tight lower bound on the marginal likelihood of DPM~\citep{HellerKA2005icml,WallachHM2010aistats}
and to be an alternative posterior inference algorithm for DPM. 
However, when evaluating the dissimilarity between nodes, one has to repeatedly compute the marginal likelihood
of clusters and careful tuning of hyperparameters are required.

In this paper, we study BHC when the underlying distributions are conjugate exponential families.
Our contributions is twofold. 
First, we derive a non-probabilistic relaxation of BHC, referred to as RBHC,
by performing {\em small variance asymptotics,} i.e., letting the variance of the underlying distribution in the model go to zero.
To this end, we use the technique inspired by the recent work \citep{KulisB2012icml,JiangK2012nips}, 
where the Gibbs sampling algorithm for DPM with conjugate exponential family was shown to approach 
a $k$-means-like hard clustering algorithm in the small variance limit.
The dissimilarity measure in RBHC is of a simpler form, compared to the one in the original BHC.
It does not require careful tuning of  hyperparameters, and yet has the flexibility of the original BHC to infer
a proper number of the clusters in data. 
It turns out to be equivalent to the dissimilarity proposed in \citep{TelgarskyM2012icml}, 
which was derived in different perspective, minimizing a cost function involving Bregman information~\citep{BanerjeeA2005jmlr}.
Second, we study the {\em reducibility}~\citep{BruynoogheM78} of the dissimilarity measure in RBHC.
If the dissimilarity is reducible, one can use \emph{nearest-neighbor chain algorithm}~\citep{BruynoogheM78} to 
build a binary tree with much smaller complexities, compared to the greedy algorithm.
The nearest neighbor chain algorithm builds a binary tree of $n$ data points in $O(n^2)$ time and $O(n)$ space,
while the greedy algorithm does in $O(n^2\log n)$ time and $O(n^2)$ space. 
We argue is that even though we cannot guarantee that the dissimilarity in RBHC is always reducible, 
it satisfies the reducibility in many cases, so it is fine to use the nearest-neighbor chain algorithm in practice to speed up 
building a binary tree using the RBHC.
We also present the conditions where the dissimilarity in RBHC is more likely to be reducible.

\section{BACKGROUND}
\label{sec:background}

We briefly review agglomerative hierarchical clustering, Bayesian hierarchical clustering, and Bregman clustering,
on which we base the development of our clustering algorithm RBHC.
Let $\bX = \{\bx_i\}_{i=1}^n$ be a set of $n$ data points. 
Denote by $[n] = \{1,\dots,n\}$ a set of indices.
A {\em partition} $\pi_n$ of $[n]$ is  a set of disjoint nonempty subsets of $[n]$ whose union is $[n]$.
The set of all possible partitions of $[n]$ is denoted by $\Pi_{n}$
For instance, in the case of $[5]=\{1,2,3,4,5\}$, an exemplary random partition that consists of three clusters is 
$\pi_5=\bigl\{ \{1\}, \{2,4\}, \{3,5\} \bigr\}$; its members are indexed by $c \in \pi_5$.
Data points in cluster $c$ is denoted by $\bX_c \defeq \{\bx_i | i \in c\}$ for $c \in \pi_n$.
Dissimilarity between $c_0 \in \pi_n$ and $c_1 \in \pi_n$ is given by $\sfd(c_0,c_1)$.

\subsection{Agglomerative Hierarchical Clustering}
\label{subsec:ahc}

Given $\bX$ and $\sfd(\cdot,\cdot)$, a common approach to building a binary tree for agglomerative hierarchical clustering 
is the greedy algorithm, where pairs of nodes are merged as one moves up the hierarchy, starting in its leaf nodes
(\textbf{\small{Algorithm~\ref{alg:greedy_ahc}}}).
A naive implementation of the greedy algorithm requires $O(n^3)$ in time since each iteration needs  $O(n^2)$
to find the pair of closest nodes and the algorithm runs over $n$ iterations.
It requires $O(n^2)$ in space to store pairwise dissimilarities. 
The time complexity can be reduced to $O(n^2\log n)$
with priority queue, and can be reduced further for some special cases; for example,
in single linkage clustering where 
\bee
\sfd(c_0,c_1) = \displaystyle{\min_{i\in c_0, j\in c_1}} \sfd(\{i\},\{j\}),
\eee
the time complexity is $O(n^2)$ since building a binary tree is equivalent to finding the minimum spanning tree in the dissimilarity graph. 
Also, in case of the centroid linkage clustering where the distance between clusters are defined as the Euclidean distance between
the centers of the clusters, one can reduce the time complexity to $O(kn^2)$ where $k \ll n$~\citep{DayWHE84joc}.

\begin{algorithm}[htp]
\small
\caption{\small Greedy algorithm for agglomerative hierarchical clustering}
\label{alg:greedy_ahc}
\begin{algorithmic}[1]
\REQUIRE $\bX = \{\bx_i\}_{i=1}^n$, $\sfd(\cdot,\cdot)$.
\ENSURE Binary tree.
\STATE Assign data points to leaves.
\STATE Compute $\sfd(\{i\},\{j\})$ for all $i,j\in [n]$.
\WHILE{the number of nodes $>$ 1}
\STATE Find a pair $(c_0,c_1)$ with minimum $\sfd(c_0,c_1)$.
\STATE Merge $c \leftarrow c_0\cup c_1$ and compute $\sfd(c, c')$ for all $c'\neq c$.
\ENDWHILE
\end{algorithmic}
\end{algorithm}

\begin{algorithm}[htp]
\small
\caption{\small Nearest neighbor chain algorithm}
\label{alg:nnchain_ahc}
\begin{algorithmic}[1]
\REQUIRE $\bX = \{\bx_i\}_{i=1}^n$, reducible $\sfd(\cdot,\cdot)$.
\ENSURE Binary tree.
\WHILE{the number of nodes $>$ 1}
\STATE Pick any $c_0$.
\STATE Build a chain $c_1 = \mathrm{nn}(c_0), c_2 = \mathrm{nn}(c_1),\dots$,
where $\mathrm{nn}(c) \defeq \argmin_{c'} \sfd(c,c')$.
Extend the chain until $c_i = \mathrm{nn}(c_{i-1})$ and $c_{i-1} = \mathrm{nn}(c_i)$.
\STATE Merge $c\leftarrow c_i\cup c_{i-1}$.
\STATE If $i \geq 2$, go to line 3 and extend the chain from $c_{i-2}$. Otherwise, go to line 2.
\ENDWHILE
\end{algorithmic}
\end{algorithm}

\subsection{Reducibility}
\label{subsec:reducibility}

Two nodes $c_0$ and $c_1$ are {\em reciprocal nearest neighbors} (RNNs) if the dissimilarity $\sfd(c_0,c_1)$
is minimal among all dissimilarities from $c_0$ to elements in $\pi_n$ and also minimal among all dissimilarities from $c_1$.
Dissimilarity $\sfd(\cdot,\cdot)$ is {\em reducible}~\citep{BruynoogheM78}, if for any  $c_0,c_1,c_2 \in \pi_n$,
\be
\label{eq:reducibility}
&\sfd(c_0,c_1) \leq \min\{\sfd(c_0,c_2),\sfd(c_1,c_2)\}\nn
& \Rightarrow \min\{\sfd(c_0,c_2),\sfd(c_1,c_2)\} \leq \sfd(c_0\cup c_1, c_2).
\ee

The reducibility ensures that if $c_0$ and $c_1$ are reciprocal nearest neighbors (RNNs),
then this pair of nodes are the closest pair that the greedy algorithm will eventually find by searching on an entire space.
Thus, the reducibility saves the effort of finding a pair of nodes with minimal distance.
Assume that $(c_0,c_1)$ are RNNs. Merging $(c_0,c_1)$ become problematic only if, for other RNNs $(c_2,c_3)$, 
merging $(c_2,c_3)$ changes the nearest neighbor of $c_0$ (or $c_1$) to $c_2\cup c_3$. 
However this does not happen since
\bee
&\sfd(c_2,c_3) \leq \min\{\sfd(c_2,c_0),\sfd(c_3,c_0)\} \\
&\Rightarrow \min\{\sfd(c_2,c_0),\sfd(c_3,c_0)\} \leq \sfd(c_2\cup c_3,c_0)\\
&\Rightarrow \sfd(c_0,c_1) \leq \sfd(c_2\cup c_3,c_0).
\eee
The nearest neighbor chain algorithm \citep{BruynoogheM78} enjoys this property and 
find pairs of nodes to merge by following paths in the nearest neighbor graph of the nodes until the paths terminate 
in pairs of mutual nearest neighbors ({\small{\textbf{Algorithm~\ref{alg:nnchain_ahc}}}}).
The time and space complexity of the nearest neighbor chain algorithm are $O(n^2)$ and $O(n)$, respectively.
The reducible dissimilarity includes those of  single linkage and Ward's method~\citep{WardJH63jasa}.

\subsection{Agglomerative Bregman Clustering}
\label{subsec:abc}

Agglomerative clustering with Bregman divergence \citep{Bregman67} as a dissimilarity measure was recently developed in
\citep{TelgarskyM2012icml}, where the clustering was formulated as the minimization of the sum of cost based on
the Bregman divergence between the elements in a cluster and center of the cluster. 
This cost is closely related with the Bregman Information used for Bregman hard clustering~\citep{BanerjeeA2005jmlr}).
In \citep{TelgarskyM2012icml}, the dissimilarity between two clusters $(c_0, c_1)$ is defined as the change of cost function 
when they are merged. As will be shown in this paper, this dissimilarity turns out to be identical to the one we derive
from the asymptotic limit of BHC. 
Agglomerative clustering with Bregman divergence showed better accuracies than traditional agglomerative hierarchical 
clustering algorithms on various real datasets~\citep{TelgarskyM2012icml}.

\subsection{Bayesian Hierarchical Clustering}
\label{subsec:bhc}

Denote by $\calT_c$ a tree whose leaves are $\bX_c$ for $c \in \pi_n$.
A binary tree constructed by BHC \citep{HellerKA2005icml} explains the generation of $\bX_c$
with two hypotheses compared in considering each merge: 
(1) the first hypothesis $\calH_c$ where all elements in $\bX_c$ were generated from a single cluster $c$;
(2) the alternative hypothesis where $\bX_c$ has two sub-clusters $\bX_{c_0}$
and $\bX_{c_1}$, each of which is associated with subtrees $\calT_{c_0}$ and $\calT_{c_1}$, respectively.
Thus, the probability of $\bX_c$ in tree $\calT_c$ is written as:
\be
\label{eq:bhc_recursive}
\lefteqn{p(\bX_c|\calT_c) = p(\calH_c)p(\bX_c|\calH_c) }\nn
&+& \{1-p(\calH_c) \}p(\bX_{c_0}|\calT_{c_0})p(\bX_{c_1}|\calT_{c_1}),
\ee
where the prior $p(\calH_c)$ is recursively defined as
\be
&\gamma_{\{i\}} \defeq \alpha,\,\,\gamma_c \defeq \alpha\Gamma(|c|)+\gamma_{c_0}\gamma_{c_1}, \\
&p(\calH_c) \defeq \alpha\Gamma(|c|)/\gamma_c,
\ee
and the likelihood of $\bX_c$ under $\calH_c$ is given by 
\be\label{eq:bhc_single}
p(\bX_c|\calH_c) \defeq \int \bigg\{\prod_{i\in c} p(\bx_i|\btheta)\bigg\} p(d\btheta).
\ee
Now, the posterior probability of $\calH_c$, which is the probability of merging $(c_0,c_1)$, is computed by Bayes rule:
\be
p(\calH_c|\bX_c,\calT_c) = \frac{p(\calH_c)p(\bX_c|\calH_c)}{p(\bX_c|\calT_c)}.
\ee

In \citep{LeeJH2014aistats}, an alternative formulation for the generative probability was proposed, which writes
the generative process via the unnormalized probabilities (potential functions):
\be
&\phi(\bX_c|\calH_c) \defeq \alpha\Gamma(|c|)p(\bX_c|\calH_c),\\
&\phi(\bX_c|\calT_c)\defeq \gamma_c p(\bX_c|\calT_c).
\ee
With these definitions, \eqref{eq:bhc_recursive} is written as
\be
\phi(\bX_c|\calT_c)= \phi(\bX_c|\calH_c) + \phi(\bX_{c_0}|\calT_{c_0})\phi(\bX_{c_1}|\calT_{c_1}),
\ee
and the posterior probability of $\calH_c$ is written as
\be
p(\calH_c|\bX_c,\calT_c) = \bigg\{1 + \frac{\phi(\bX_{c_0}|\calT_{c_0})\phi(\bX_{c_1}|\calT_{c_1})}{\phi(\bX_c|\calH_c)}\bigg\}^{-1}.
\ee
One can see that the ratio inside behaves as the dissimilarity between $c_0$ and $c_1$:
\be\label{eq:bhc_dis_1}
\sfd(c_0,c_1)\defeq \frac{\phi(\bX_{c_0}|\calT_{c_0})\phi(\bX_{c_1}|\calT_{c_1})}{\phi(\bX_c|\calH_c)}.
\ee
Now, building a binary tree follows {\small \textbf{Algorithm}~\ref{alg:greedy_ahc}} with the distance in (\ref{eq:bhc_dis_1}). 
Beside this, BHC has a scheme to determine the number of clusters. 
It was suggested in \citep{HellerKA2005icml} that the tree can be cut at points $p(\calH_c|\bX_c,\calT_c) < 0.5$.
It is equivalent to say that the we stop the algorithm if the minimum over $\sfd(c_0,c_1)$ is greater than 1.
Note that once the tree is cut, the result contains forests, each of which involves a cluster.

BHC is closely related to the marginal likelihood of DPM; actually, the prior $p(\calH_c)$ comes from
the predictive distribution of DP prior. Moreover, it was shown that computing $\phi(\bX_c|\calT_c)$ to build a tree naturally
induces a lower bound on the marginal likelihood of DPM, $\pdpm(\bX_c)$~\citep{HellerKA2005icml}:
\be
\frac{\Gamma(\alpha)}{\Gamma(\alpha+n)}\phi(\bX_c|\calT_c)\leq \pdpm(\bX_c).
\ee
Hence, in the perspective of the posterior inference algorithm for DPM, building tree in BHC is equivalent to computing the approximate
marginal likelihood. Also, cutting the tree at the level where $\sfd(c_0,c_1) > 1$ corresponds finding the MAP clustering of $\bX$.

In \citep{HellerKA2005icml}, the time complexity was claimed to be $O(n^2)$. 
However, this does not count the complexity required to find a pair with the smallest dissimilarity via sorting. 
For instance, with a sorting algorithm using priority queues, BHC requires $O(n^2\log n)$ in time.

The dissimilarity is very sensitive to tuning the hyperparameters involving the distribution over parametes$p(\btheta)$ 
required to compute $p(\bX_c|\calH_c)$.
An EM-like iterative algorithm was proposed in \citep{HellerKA2005icml} to tune the hyperparameters, but the repeated
execution of the algorithm is infeasible for large-scale data.

\subsection{Bregman Diverences and Exponential Families}
\label{subsec:brdiv_and_ef}

\begin{defn}
(\citep{Bregman67}) Let $\varphi$ be a strictly convex differentiable function defined on a convex set. 
Then, the Bregman divergence, $\sfB_\varphi(\bx,\by)$, is defined as
\be
\sfB_\varphi(\bx,\by) = \varphi(\bx)-\varphi(\by) - \ip{\bx-\by,\nabla\varphi(\by)}.
\ee
\end{defn}

Various divergences belong to the Bregman divergence.
For instance, Euclidean distance or KL divergence is Bregman divergence, when 
$\varphi(\bx)= \frac{1}{2}\bx^{\top} \bx$ or $\varphi(\bx) = \bx \log \bx$, respectively.

The exponential family distribution over $\bx\in\bbr^d$ with 
natural parameter $\btheta\in\bTheta$ is of the form:
\be
p(\bx|\btheta) = \exp \Big\{ \ip{t(\bx),\btheta} - \psi(\btheta) - h(\bx) \Big\},
\ee
where $t(\bx)$ is sufficient statistics, $\psi(\btheta)$ is a log-partition function, 
and $\exp\{- h(\bx)\}$ is a base distribution. 
We assume that $p(\bx|\btheta)$ is \emph{regular} ($\bTheta$ is open) and $t(\bx)$ is \emph{minimal} ($\nexists \ba \in\bbr^d$ s.t. $\forall \bx,\,\, \ip{\ba,t(\bx)} = \mathrm{const}$).
Let $\varphi(\bmu)$ be the convex conjugate of $\psi$:
\be
\varphi(\bmu) = \sup_{\theta\in\Theta} \Big\{\ip{\bmu,\btheta}-\psi(\btheta) \Big\}.
\ee
Then, the Bregman divergence and the exponential family has the following relationship:

\begin{thm}
\citep{BanerjeeA2005jmlr}
Let $\varphi(\cdot)$ be the conjugate function of $\psi(\cdot)$.
Let $\btheta$ be the natural parameter and $\bmu$ be the corresponding expectation parameter, i.e., 
$\bmu = \bbe[t(\bx)] = \nabla\psi(\btheta)$.
Then $p(\bx|\btheta)$ is uniquely expressed as
\be
p(\bx|\btheta) & = & \exp \Big\{-\sfB_\varphi(t(\bx),\bmu) \Big\} \nonumber \\
& & \exp \Big\{\varphi(t(\bx))-h(\bx) \Big\}.
\ee
\end{thm}

The conjugate prior for $\btheta$ has the form:
\be
p(\btheta|\nu,\btau) = \exp \Big\{ \ip{\btau,\btheta} - \nu\psi(\btheta) - \xi(\nu,\btau) \Big\}.
\ee
$p(\btheta|\nu,\btau)$ can also be expressed with the Bregman divergence:
\be
\lefteqn{p(\btheta|\nu,\btau) = \exp\{-\nu\sfB_\varphi(\btau/\nu,\bmu)\}} \nn
&& \times \exp\Big\{\nu\varphi(\btau/\nu)-\xi(\nu,\btau)\Big\}.
\ee

\subsection{Scaled Exponential Families}
\label{subsec:scaled_ef}

Let $\widetilde\btheta = \beta\btheta$, and $\widetilde\psi(\widetilde\btheta) = \beta\psi(\widetilde\btheta/\beta) = \beta\psi(\btheta)$.
The scaled exponential family with scale $\beta$ is defined as follows~\citep{JiangK2012nips}:
\be
p(\bx|\widetilde\btheta) &=& \exp \Big\{ \ip{t(\bx),\widetilde\btheta}-\widetilde\psi(\widetilde\btheta)-h_\beta(\bx) \Big\}\nn
&=&\exp \Big\{ \beta \ip{t(\bx),\btheta}-\beta\psi(\btheta)-h_\beta(\bx) \Big\}.
\ee
For this scaled distribution, the mean $\bmu$ remains the same, and the covariance $\bSigma$ becomes 
$\bSigma/\beta$ ~\citep{JiangK2012nips}.
Hence, the distribution is more concentrated around its mean. The scaled distribution in the Bregman divergence form is
\be
p(\bx|\widetilde\btheta) & = & \exp \Big\{-\beta\sfB_\varphi(t(\bx),\bmu) \Big\} \nonumber \\
& & \exp \Big\{\beta\varphi(t(\bx))-h_\beta(\bx) \Big\}.
\ee
According to $p(\bx|\widetilde\btheta)$, $p(\widetilde\btheta|\widetilde\btau,\widetilde\nu)$ is defined with $\widetilde\btau = \btau/\beta$, $\widetilde\nu = \nu/\beta$.
Actually, this yields the same prior as non-scaled distribution.
\be
p(\widetilde\btheta|\widetilde\nu,\widetilde\btau) &=& 
\exp\Big\{\ip{\widetilde\btheta,\widetilde\btau}-
\widetilde\nu\widetilde\psi(\widetilde\btheta)-\xi_\beta(\widetilde\nu,\widetilde\btau)\Big\}\nn
&=& \exp \Big\{ \ip{\btheta,\btau}-\nu\psi(\btheta)-\xi(\nu,\btau) \Big\}.
\ee

\section{MAIN RESULTS}
\label{sec:main}

We present the main contribution of this paper.
From now on, we assume that the likelihood and prior in Eq.~\eqref{eq:bhc_single} 
are scaled exponential families defined in Section \ref{subsec:scaled_ef}.

\subsection{Small-Variance Asymptotics for BHC}
\label{subsec:sva_bhc}

The dissimilarity in BHC can be rewritten as follows:
\be
\label{eq:bhc_dis_2}
\sfd(c_0,c_1) &=& \frac{\phi(\bX_{c_0}|\calT_{c_0})\phi(\bX_{c_1}|\calT_{c_1})}{\phi(\bX_c | \calH_c)}\nn
&=& \frac{\alpha\Gamma(|c_0|)\Gamma(|c_1|)p(\bX_{c_0}|\calH_{c_0})p(\bX_{c_1}|\calH_{c_1})}{\Gamma(|c|)p(\bX_c|\calH_c)}\nn
& \times & \Big\{ 1 + \sfd(c_{00},c_{01}) \Big\} \Big\{1+\sfd(c_{10},c_{11}) \Big\},
\ee
where $c_0 = c_{00 }\cup c_{01}$ and $c_1 = c_{10} \cup c_{11}$. 
We first analyze the term $p(\bX_c|\calH_c)$, as in \citep{JiangK2012nips}.
\be
\label{eq:integral}
\lefteqn{p(\bX_c|\calH_c) = \disint \bigg\{\prod_{i\in c} p(\bx_i|\widetilde\btheta)\bigg\} p(\widetilde\btheta|\nu,\btau)d\widetilde\btheta}\nn
&=&\beta^d\disint\exp \bigg\{ \bigg\langle \btheta, \btau + \beta\sum_{i\in c} t(\bx_i)\bigg\rangle - (\nu+\beta|c|)\psi(\btheta)\nn
& & -\sum_{i\in c} h_{\beta}(\bx_i) - \xi(\nu,\btau)\bigg\} d\btheta\nn
&=&\beta^d \exp\bigg\{(\nu+\beta|c|)\varphi(\bmu_c)-\sum_{i\in c} h_{\beta}(\bx_i) - \xi(\nu,\btau)\bigg\}\nn
& & \times \int \exp\{-(\nu+\beta|c|)\sfB_\varphi(\bmu_c,\bmu)\}d\btheta,
\ee
where
\be
\bmu_c \defeq \frac{\btau+\beta\sum_{i\in c} t(\bx_i)}{\nu+\beta|c|}.
\ee
Note that $\bmu = \nabla\psi(\btheta)$ is a function of $\btheta$. The term inside the integral of Eq.~\eqref{eq:integral} has a local minimum
at $\bmu = \bmu_c$, and thus can be approximated by Laplace's method:
\be
\label{eq:integral_approx}
&= \beta^d \exp\bigg\{ (\nu+\beta |c|)\varphi(\bmu_c) - \displaystyle{\sum_{i\in c}} h_{\beta}(\bx_i) - \xi(\nu,\btau)\bigg\}\nn
&  \bigg(\dfrac{2\pi}{\nu+\beta|c|}\bigg)^{\frac{d}{2}}\bigg|\dfrac{\partial^2\sfB_\varphi(\bmu_c,\bmu)}{\partial\btheta\partial\btheta\tr}\bigg|^{-\frac{1}{2}}_{\bmu=\bmu_c}\Big\{1+O(\beta^{-1})\Big\}.
\ee
It follows from this result that, as $\beta \rightarrow \infty$,  
the asymptotic limit of dissimilarity $\sfd(c_0,c_1)$ in \eqref{eq:bhc_dis_2} is given by
{\small
\bee
\label{eq:lim}
\lefteqn{\lim_{\beta\to\infty}\frac{\alpha\Gamma(|c_0|)\Gamma(|c_1|)p(\bX_{c_0}|\calH_{c_0})p(\bX_{c_1}|\calH_{c_1})}{\Gamma(|c|)p(\bX_c|\calH_c)}}\\
&\propto& \hspace*{-.1in} \lim_{\beta\to\infty}\alpha\beta^{\frac{d}{2}}\exp \Bigl\{ \beta(|c_0|\varphi(\bmu_{c_0}) 
+ |c_1|\varphi(\bmu_{c_1}) - |c|\varphi(\bmu_c)) \Bigr\}\nonumber,
\eee}
Let $\alpha = \beta^{-\frac{d}{2}}\exp(-\beta\lambda)$, then we have
\bee
= \lim_{\beta\to\infty}\exp \Big\{\beta(|c_0|\varphi(\bmu_{c_0}) + |c_1|\varphi(\bmu_{c_1}) - |c|\varphi(\bmu_c) - \lambda) \Big\}.
\eee
As $\beta \to \infty$, the term inside the exponent converges to
\be
|c_0|\varphi(\bar t_{c_0}) + |c_1|\varphi(\bar t_{c_1}) - |c|\varphi(\bar t_{c})-\lambda,
\ee
where
\be
\bar t_c \defeq \frac{1}{|c|} \sum_{i\in c} t(\bx_i),
\ee
and this is the average of sufficient statistics for cluster $c$. 
With this result, we define a new dissimilarity $\sfd_\star(c_0,c_1)$ as
\be
\label{eq:bhc_dis_star}
\sfd_\star(c_0,c_1) \hspace*{-.1in} & \defeq &  \hspace*{-.1in} |c_0|\varphi(\bar t_{c_0}) + |c_1|\varphi(\bar t_{c_1}) - |c|\varphi(\bar t_{c}) \nn
\hspace*{-.1in} &=& \hspace*{-.1in} |c_0|\varphi(\bar t_{c_0}) + |c_1|\varphi(\bar t_{c_1})\nn
\hspace*{-.1in} & - & \hspace*{-.1in} (|c_0|+|c_1|)\varphi\bigg(\frac{|c_0|\bar t_{c_0}+|c_1|\bar t_{c_1}}{|c_0|+|c_1|}\bigg).
\ee
Note that $\sfd_\star(\cdot,\cdot)$ is always positive since $\varphi$ is convex. 
If $\sfd_\star(c_0,c_1) \geq \lambda$, the limit Eq.~\eqref{eq:lim} diverges to $\infty$,
and converges to zero otherwise. When $|c|=1$, Eq.~\eqref{eq:lim} is the same as the limit of the dissimilarity $\sfd(c_0,c_1)$, 
and thus the dissimilarity diverges when $\sfd_\star(c_0,c_1) \geq \lambda$ and converges otherwise. When $|c|>1$, assume that the dissimilarities of 
children $\sfd(c_{00},c_{01})$ and $\sfd(c_{10},c_{11})$ converges to zero. From Eq.~\eqref{eq:bhc_dis_2}, we can easily see that
$\sfd(c_0,c_1)$ converges only if $\sfd_\star(c_0,c_1) < \lambda$. In summary, 
\be\label{eq:bhc_asymp_thres}
\lim_{\beta\to\infty} \sfd(c_0,c_1) = \bigg\{ \begin{array}{ll} 0 & \textrm{ if } \sfd_\star(c_0,c_1) < \lambda,\\ \infty & \textrm { otherwise.}\end{array}.
\ee
In similar way, we can also prove the following:
\be\label{eq:bhc_asymp_compare}
\lim_{\beta\to\infty} \frac{\sfd(c_0,c_1)}{\sfd(c_2,c_3)} = \bigg\{ \begin{array}{ll} 0 & \textrm{ if } \sfd_\star(c_0,c_1) < \sfd_\star(c_2,c_3),\\ \infty & \textrm { otherwise.}\end{array},
\ee
which means that comparing two dissimilarities in original BHC is equivalent to comparing the new dissimilarities $\sfd_\star(\cdot,\cdot)$,
and we can choose the next pair to merge by comparing $\sfd_\star(\cdot,\cdot)$ instead of $\sfd(\cdot,\cdot)$.

With Eqs.~\eqref{eq:bhc_asymp_thres} and \eqref{eq:bhc_asymp_compare}, 
we conclude that when $\beta\to\infty$, BHC reduces to {\textbf{Algorithm~\ref{alg:greedy_ahc}}}
with dissimilarity measure $\sfd_\star(\cdot,\cdot)$ and threshold $\lambda$, 
where the algorithm terminates when the minimum $\sfd_\star(\cdot,\cdot)$ exceeds $\lambda$.

On the other hand, a simple calculation yields
\bee
\sfd_\star(c_0,c_1) = |c_0| \sfB_\varphi(\bar t_{c_0}, \bar t_c) + |c_1|\sfB_\varphi(\bar t_{c_1},\bar t_c),
\eee
which is exactly same as the dissimilarity proposed in \citep{TelgarskyM2012icml}.
Due to the close relationship between exponential family and the Bregman divergence,
the dissimilarities derived from two different perspective has the same form.

As an example, assume that $p(\bx|\btheta) = \calN(\bx|\bmu,\sigma^2\bI)$ and $p(\bmu) = \calN(\bmu|0,\rho^2\bI)$. We have $\varphi(\bx) = \norm{\bx}^2/(2\sigma^2)$ and
\be\label{eq:ward}
\sfd_\star(c_0,c_1) = \frac{|c_0||c_1|\norm{\bx_{c_0}-\bx_{c_1}}^2}{2\sigma^2(|c_0|+|c_1|)},
\ee
which is same as the Ward's merge cost~\citep{WardJH63jasa}, except for the constant $1/(2\sigma^2)$. Other examples
can be found in \citep{BanerjeeA2005jmlr}.

Note that $\sfd_\star(\cdot,\cdot)$ does not need hyperparameter tunings, since
the effect of prior $p(\btheta)$ is ignored as $\beta\to 0$. This provides a great advantage
over BHC which is sensitive to the hyperparameter settings.

\textbf{Smoothing}: In some particular choice of $\varphi$, the singleton clusters may have degenerate values~\citep{TelgarskyM2012icml}. 
For example, when $p(\bx|\btheta) = \mathrm{Mult}(\bx|m,q)$, the function $\varphi(\bx) = \sum_{j=1}^d x_j\log(x_j/m)$ has degenerate values
when $x_j = 0$. To handle this, we use the smoothing strategy proposed in \citep{TelgarskyM2012icml}; instead of the
original function $\varphi(\bx)$, we use the smoothed functions $\varphi_0(\bx)$ and $\varphi_1(\bx)$ defined as follows:
\be
&\varphi_0(\bx) \defeq \varphi((1-\alpha)\bx + \alpha \bgamma), \\
&\varphi_1(\bx) \defeq \varphi(\bx + \alpha \bgamma),
\ee
where $\alpha\in(0,1)$ be arbitrary constant and $\bgamma$ must in the relative interior of the domain of $\varphi$.
In general, we use $\varphi_0(\bx)$ as a smoothed function, but we can also use $\varphi_1(\bx)$ when the domain of $\varphi$ is a convex cone.

\textbf{Heuristics for choosing $\lambda$}: As in \citep{KulisB2012icml}, we choose the threshold value $\lambda$.
Fortunately, we found that the clustering accuracy was not extremely sensitive to the choice of $\lambda$; merely
selecting the scale of $\lambda$ could result in reasonable accuracy. There can be many simple heuristics,
and the one we found effective is to use the $k$-means clustering. With the very rough guess on the desired number
of clusters $\widetilde k$, we first run the $k$-means clustering (with Euclidean distance) with $k = a\widetilde k$ (we fixed
$a = 4$ for all experiments). Then, $\lambda$ was set to the average value of dissimilarities $\sfd_\star(\cdot,\cdot)$
between the all pair of $k$ centers.

\subsection{Reducibility of $\sfd_\star(\cdot,\cdot)$}
\label{subsec:reducible}

The relaxed BHC with small-variance asymptotics still has the same complexities to BHC.
 If we can show that $\sfd_\star(\cdot,\cdot)$ is reducible, we can reduce the complexities by adapting the nearest neighbor chain algorithm.
Unfortunately, $\sfd_\star(\cdot,\cdot)$ is not reducible in general (one can easily find counter-examples for some distributions). 
However, we argue that $\sfd_\star(\cdot,\cdot)$ is reducible in many cases,
and thus applying the nearest neighbor chain algorithm as if $\sfd_\star(\cdot,\cdot)$ is reducible
does not degrades the clustering accuracy. In this section, we show the reason by analyzing  $\sfd_\star(\cdot,\cdot)$.

At first, we investigate a term inside the dissimilarity:
\be
f(\bar t_c) \defeq |c|\varphi(\bar t_c).
\ee
The second-order Taylor expansion of this function around the mean $\bmu$ yields:
\be
\label{eq:taylor}
&f(\bar t_c) = |c|\varphi(\bmu) + |c| \varphi^{(1)}(\bmu)\tr (\bar t_c-\bmu)\nn
&  + \Delta_\varphi(\bar t_c,\bmu) + \epsilon_\varphi(\bar t_c,\bmu),
\ee
where $\varphi^{(k)}$ is the $k$th order derivative of $\varphi$, and
\be
\Delta_\varphi(\bar t_c) \defeq \frac{|c|}{2}(\bar t_c-\bmu)\tr\varphi^{(2)}(\bmu)(\bar t_c-\bmu), \\
\epsilon_\varphi(\bar t_c) \defeq |c|\sum_{|\balpha|=3} \frac{\partial^{\balpha} \varphi(\bnu)}{\balpha!}(\bar t_c-\bmu)^{\balpha}.
\ee
Here, $\balpha$ is the multi-index notation, and $\bnu = \bmu + k(\bar t_c-\bmu)$ for some $k\in(0,1)$. The term $\Delta_\varphi(\bar t_c)$ plays an important role in analyzing the
reducibility of $\sfd_\star(\cdot,\cdot)$. To bound the error term $\epsilon_\varphi(\bar t_c)$,
we assume that $|\varphi^{(3)}|\leq M$\footnote{This assumption holds for the most of distributions we will discuss (if properly smoothed), but not holds in general.}. As earlier, assume that $\bar t_c$ is a average of $|c|$ observations generated from the same
scaled-exponential family distribution:
\be
\bx_1,\dots, \bx_n  \isim p(\cdot | \beta,\btheta),\quad \bar t_c = \frac{1}{|c|}\sum_{i\in c} t(\bx_i).
\ee
By the property of the log-partition function of the exponential family distribution, we get the following results:
\be\label{eq:expected_error}
\bbe[\epsilon_\varphi(\bar t_c)] = \frac{1}{\beta^2 |c|}\sum_{|\balpha|=3} \frac{\partial^{\balpha}\varphi(\bnu)\partial^{\balpha} \psi(\btheta)}{\balpha!}.
\ee
One can see that the expected error converges to zero as $\beta^2|c|\to\infty$. Also, it can be shown that
the expectation of the ratio of two terms converges to zero as $\beta \to \infty$:
\be
\lim_{\beta\to\infty} \bbe\bigg[\frac{\epsilon_\varphi(\bar t_c)}{\Delta_\varphi(\bar t_c)}\bigg] = 0,
\ee
which means that $\Delta_\varphi(\bar t_c)$ asymptotically dominates $\epsilon_\varphi(\bar t_c)$ (detailed derivations are given in the supplementary material).
Hence, we can safely approximate $f(\bar t_c)$ up to second order term.

Now, let $c_0$ and $c_1$ be clusters belong to the same super-cluster (i.e. $\bX_{c_0}$ and $\bX_{c_1}$ were
generated from the same mean vector $\bmu$). We don't need to investigate the case where the pair belong to a different cluster,
since then they will not be merged anyway $(\sfd_\star(\cdot,\cdot) \geq \lambda)$ in our algorithm. By the independence, $\bbe[\bar t_{c_0}] = \bbe[\bar t_{c_1}] = \bbe[\bar t_{c_0\cup c_1}] = \bmu$. Applying the approximation~\eqref{eq:taylor}, we have
\be
\lefteqn{\sfd_\star(c_0,c_1) \approx \Delta_\varphi(\bar t_{c_0}) + \Delta_\varphi(\bar t_{c_1}) - \Delta_\varphi(\bar t_{c_0\cup c_1})  }\nn
& = \frac{|c_0||c_1|}{2(|c_0|+|c_1|)}(\bar t_{c_0}-\bar t_{c_1})\tr \varphi^{(2)}(\bmu) (\bar t_{c_0}-\bar t_{c_1}).
\ee
This approximation, which we will denote as $\widetilde\sfd_\star(c_0,c_1)$, is a generalization of the Ward's cost~\eqref{eq:ward} from Euclidean distance to Mahalanobis distance with matrix $\varphi^{(2)}(\bmu)$ (note that this approximation is exact for the spherical Gaussian case). 
More importantly, $\widetilde\sfd_\star(c_0,c_1)$ is reducible.

\begin{thm}
\bee
&\widetilde\sfd_\star(c_0,c_1) \leq \min \Big\{\widetilde\sfd_\star(c_0,c_2),\widetilde\sfd_\star(c_1,c_2) \Big\}\nn
& \Rightarrow \min \Big\{\widetilde\sfd_\star(c_0,c_2),\widetilde\sfd_\star(c_1,c_2) \Big\} \leq \widetilde\sfd_\star(c_0\cup c_1, c_2).
\eee
\end{thm}
\begin{proof}
For the Ward's cost, the following Lance-Williams update formula~\citep{LanceGN67tcj} holds for $\widetilde\sfd_\star(\cdot,\cdot)$:
\be
\lefteqn{\widetilde\sfd_\star(c_0\cup c_1,c_2)=\frac{(|c_0|+|c_2|)\widetilde\sfd_\star(c_0,c_2)}{|c_0|+|c_1|+|c_2|}}\nn
& &  + \frac{(|c_1|+|c_2|)\widetilde\sfd_\star(c_1,c_2)-|c_2|\widetilde\sfd_\star(c_0,c_1)}{|c_0|+|c_1|+|c_2|}.
\ee
Hence, by the assumption, we get
\bee
\widetilde\sfd_\star(c_0\cup c_1,c_2) \geq \min \Big\{\widetilde\sfd_\star(c_0,c_2),\widetilde\sfd_\star(c_1,c_2) \Big\}.
\eee
\end{proof}

\begin{algorithm}[t!]
\small
\caption{\small Nearest neighbor chain algorithm for BHC with small-variance asymptotics}
\label{alg:nnchain_ahc_thres}
\begin{algorithmic}[1]
\REQUIRE $\bX = \{\bx_i\}_{i=1}^n$, $\sfd_\star(\cdot,\cdot)$, $\lambda$.
\ENSURE A clustering $C$.
\STATE Set $R = [n]$ and $C = \varnothing$.
\WHILE{$R \neq \varnothing$}
\STATE Pick any $c_0\in R$.
\STATE Build a chain $c_1 = \mathrm{nn}(c_0), c_2 = \mathrm{nn}(c_1),\dots$,
where $\mathrm{nn}(c) \defeq \argmin_{c'} \sfd_star(c,c')$.
Extend the chain until $c_i = \mathrm{nn}(c_{i-1})$ and $c_{i-1} = \mathrm{nn}(c_i)$.
\STATE Remove $c_i$ and $c_{i-1}$ from $R$.
\IF{$\sfd_\star(c_{i-1},c_i)<\lambda$}
\STATE Add $c = c_i\cup c_{i-1}$ to $R$.
\STATE If $i \geq 2$, go to line 3 and extend the chain from $c_{i-2}$. Otherwise, go to line 2.
\ELSE
\STATE Add $c_i$ and $c_{i-1}$ to $C$.
\ENDIF
\ENDWHILE
\end{algorithmic} 
\end{algorithm}

As a result, the dissimilarity $\sfd_\star(\cdot,\cdot)$ is reducible provided that the Taylor's approximation \eqref{eq:taylor} is accurate.
In such a case, one can apply the nearest-neighbor chain algorithm with $\sfd_\star(\cdot,\cdot)$,
treating 
as if it is reducible to build a binary tree in $O(n^2)$ time and $O(n)$ space. Unlike
\textbf{Algorithm~\ref{alg:nnchain_ahc}}, we have a threshold $\lambda$ to determine the number of clusters, 
and we present a slightly revised algorithm~({\textbf{Algorithm~\ref{alg:nnchain_ahc_thres}}). Note again
that the revised algorithm generates forests instead of trees.

\begin{table}
	\small
	\centering
	\caption{Average value of the exact dissimilarity $\sfd_\star$, approximate dissimilarity $\widetilde\sfd_\star$
		and relative error, and number of not-reducible case among 100,000 trials.}
	\begin{tabular}{|c|c|c|c|c|}
		\hline		& $\bbe[\sfd_\star]$  & $\bbe[\widetilde{\sfd}_\star]$ & \parbox{1.2cm}{average relative error ($\%$)} & \parbox{1.2cm}{ not\\ reducible $/ 100k$ }  \\ 
		\hline Poisson & 0.501 & 0.500 & 1.577 & 0 \\ 
		Multinomial & 3.594  & 3.642 & 4.675 & 87 \\ 
		Gaussian & 8.739  & 8.717 & 7.716 & 48 \\ 
		\hline
	\end{tabular} 
		\label{table:reducibility}
\end{table}

\begin{figure*}
	\centering
	\includegraphics[width = 0.4\linewidth]{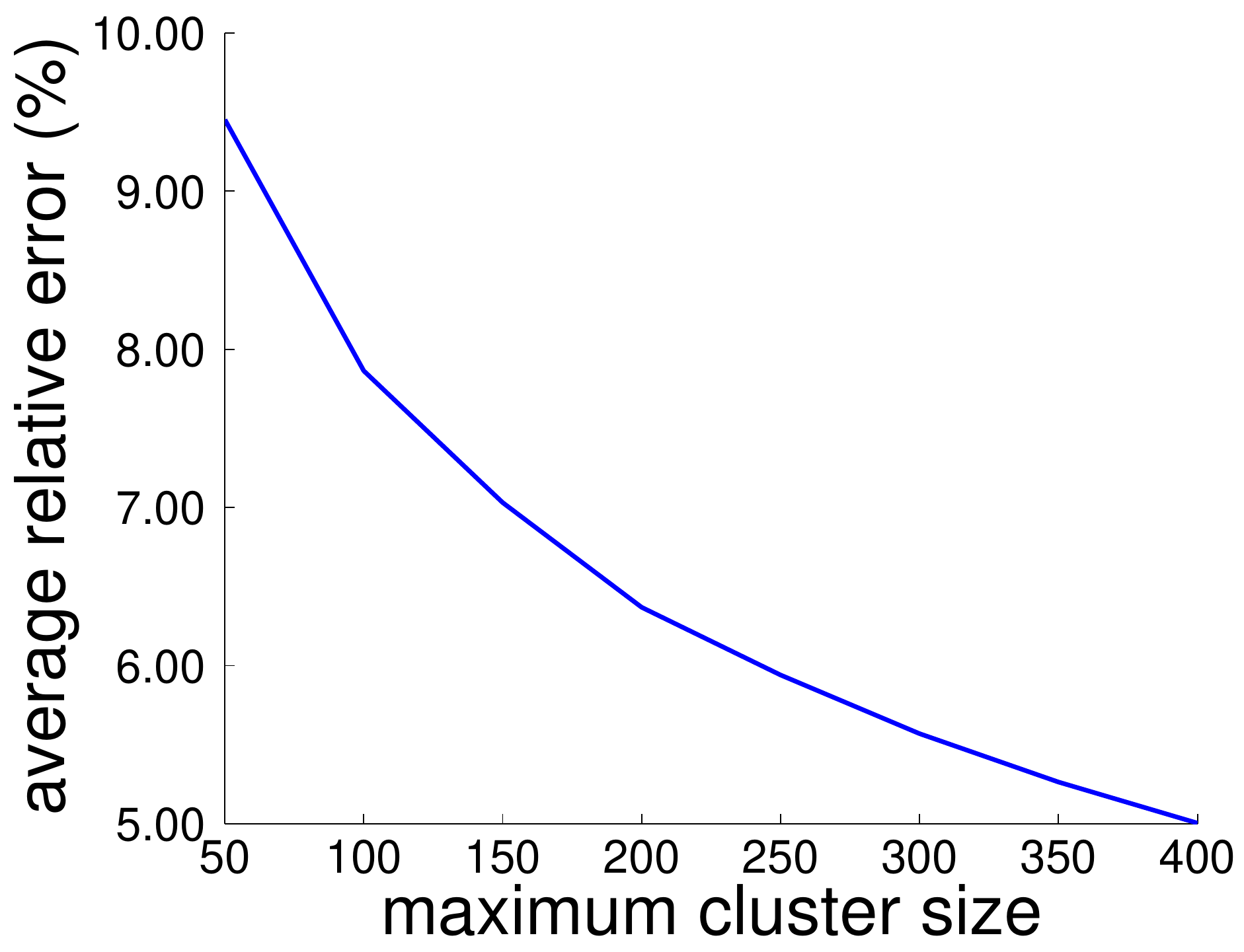}
	\includegraphics[width = 0.4\linewidth]{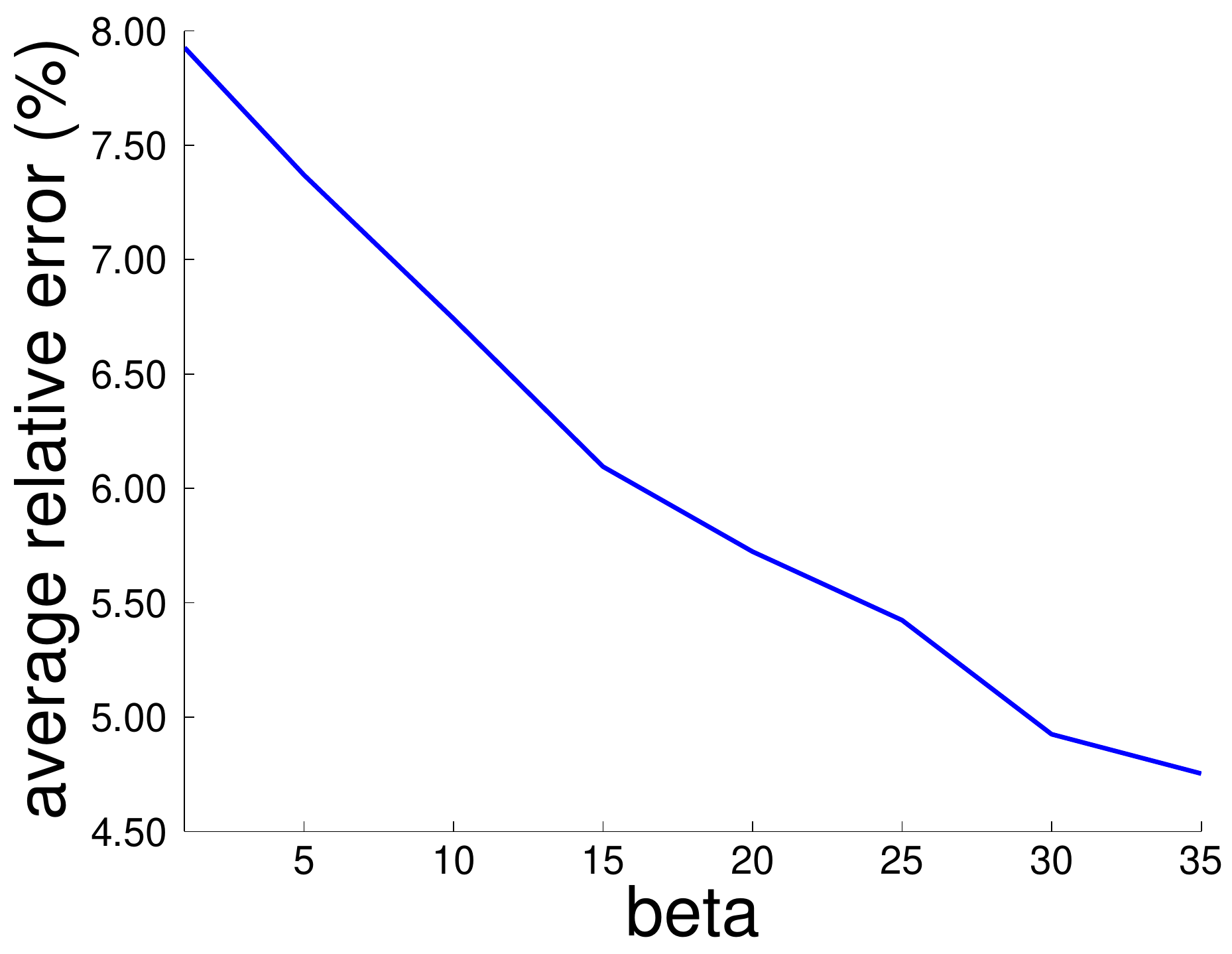}
	\caption{Average relative error vs maximum cluster size and scale factor $\beta$.}
	\label{fig:reduciblllity}
\end{figure*}

\begin{table*}
	\small
	\centering
	\caption{Average adjusted Rand index values for randomly generated datasets. Best ones are marked as bold face.}
	\begin{tabular}{|c|c|c|c|c|c|c|}
		\hline  & Poisson ($k$) & Poisson ($2k$) & multinomial ($k$) & multinomial ($2k$) & Gaussian ($k$) & Gaussian ($2k$) \\ 
		\hline single linkage & 0.091 (0.088) & 0.015 (0.030) & 0.000 (0.000) & 0.000 (0.000) & 0.270 (0.280)  & 0.057 (0.055) \\ 
		complete linakge & 0.381 (0.091) & 0.263 (0.059) & 0.266 (0.144) & 0.090 (0.025) & 0.565 (0.166) & 0.475 (0.141) \\ 
		Ward's method & 0.465 (0.119) & 0.273 (0.049) & 0.770 (0.067) & 0.564 (0.055) & 0.779 (0.145) & 0.763 (0.122) \\ 
		RBHC-greedy & \textbf{0.469} (0.112) & \textbf{0.290} (0.056) & 0.870 (0.046) & 0.733 (0.028) & \textbf{0.875} (0.087) & \textbf{0.883} (0.067) \\ 
		RBHC-nnca & \textbf{0.469} (0.112)  & \textbf{0.290} (0.056) & 0.865 (0.045) & 0.736 (0.040) & \textbf{0.875} (0.087) & \textbf{0.883} (0.067) \\ 
		BHC & 0.265 (0.080) & 0.134 (0.052) & \textbf{0.907} (0.069) & \textbf{0.894} (0.044) & 0.863 (0.109) & 0.860 (0.108) \\
		\hline 
	\end{tabular} 
	\label{table:synthetic_ari}
\end{table*}

\section{EXPERIMENTS}
\label{sec:experiments}

\subsection{Experiments on Synthetic Data}
\label{subsec:synthetic_experiments}

{\textbf{Testing the reducibility of $\sfd_\star(\cdot,\cdot)$}}: 
We tested the reducibility of $\sfd_\star(\cdot,\cdot)$ empirically. 
We repeatedly generated the three clusters $c_0, c_1$ and $c_2$ from
the exponential family distributions, and counted the number of cases where the dissimilarities between those
clusters are not reducible. We also measured the average value of the relative error to support our arguments 
in Section~\ref{subsec:reducible}. 
We tested three distributions; Poisson, multinomial and Gaussian.
At each iteration, we first sampled the size of the clusters $|c_0|$, $|c_1|$ and $|c_2|$
from $\mathrm{Unif}([20,100])$. Then we sampled the three clusters from one of the three distributions, 
and computed $\sfd_\star(c_0,c_1), \sfd_\star(c_1,c_2)$ and $\sfd_\star(c_0, c_2)$. We then first checked
whether these three values satisfy the reducibility condition~\eqref{eq:reducibility} (for example,
if $\sfd_\star(c_0,c_2)$ is the smallest, we checked if $\sfd_\star(c_0\cup c_2, c_1) \geq \min\{ \sfd_\star(c_0,c_1),
\sfd_\star(c_2, c_1)\}$). Then, for $\sfd_\star(c_0,c_1)$, we computed the approximate value $\widetilde\sfd_\star(c_0,c_1)$
and measured the relative error
\be
2\times\left|\frac{ \sfd_\star(c_0,c_1)-\widetilde{\sfd}_\star(c_0,c_1) }{\sfd_\star(c_0,c_1)+\widetilde{\sfd}_\star(c_0,c_1)}\right|.
\ee
We repeated this process for $100,000$ times for the three distributions and measured the average values.
For Poisson distribution, we sampled the mean $\rho \sim\mathrm{Gamma}(2, 0.05)$ and sampled the
data from $\mathrm{Poisson}(\rho)$. We smoothed the function $\varphi(\bx) = \bx\log(\bx)-\bx$ as $\varphi(\bx+0.01)$
to prevent degenerate function values. For multinomial distribution, we tested the case where the dimension $d$ is $10$
and the number of trials $m$ is 5. We sampled the parameter $\bq\sim\mathrm{Dir}(5\cdot\mathbf{1}_{d})$ where $\mathbf{1}_d$
is $d$-dimensional one vector, and sampled the data from $\mathrm{Mult}(\bq)$. We smoothed the function $\varphi(\bx) = \sum_{j=1}^d x_j \log(x_j/m)$ as $\varphi(0.9\bx + 0.1 m \mathbf{1}_d / d)$. For Gaussian, we
tested with $d=10$, and sampled the mean and covariance $\bmu,\bSigma \sim \calN(0,(0.08\cdot\bLambda)^{-1})\calW (d+2, \bPsi)$ ($\bPsi = \bA \bA\tr + d\bI$ where $\bA$ was sampled from unit Gaussian). We smoothed
the function $\varphi(\bx,\bX) = - \frac{1}{2}\log\det(\bX-\bx\bx\tr)$ as $-\frac{1}{2}\log\det(\bX-\bx\bx\tr + 0.01\bI)$.

The result is summarized in Table~\ref{table:reducibility}. the generated dissimilarities were reducible in most case, as expected. 
The relative error was small, which supports our arguments of the reason why $\sfd_\star(\cdot,\cdot)$ is reducible with high probability. 
We also measured the change of average relative error by controlling two factors; the maximum cluster size $n_{\mathrm{max}}$ 
and variance scale factor $\beta$. We plotted the average relative error of Gaussian distribution by changing those two factors, 
and the relative error decreased as predicted (Figure~\ref{fig:reduciblllity}).

\textbf{Clustering synthetic data}: We evaluated the clustering accuracies of original BHC (BHC),
BHC in small-variance limit with greedy algorithm (RBHC-greedy), BHC in small-variance limit with nearest neighbor chain
method (RBHC-nnca), single linkage, complete linkage, and Ward's method. We generated the datasets from Poisson, 
multinomial and Gaussian distribution. We tested two types of data; 1,000 elements with 6 clusters and 2,000 elements with 12 clusters.
For Poisson distribution, each mixture component was generated from $\mathrm{Poisson}(\rho)$ with $\rho\sim\mathrm{Gamma}(2,0.05)$ 
for both datasets. For multinomial distribution, we set $d=20$ and $m=10$
for 1,000 elements dataset, and set $d=40$ and $m=10$ for 2,000 elements datasets. For both dataset, we sampled
the parameter $\bq\sim\mathrm{Dir}(0.5\cdot\mathbf{1}_d)$. For Gaussian case, we set $d=3$ for 1,000 elements dataset
and set $d=6$ for 2,000 elements dataset. We sampled the parameters from $\calN(0,(0.08\cdot\Lambda)^{-1})\calW(\bPsi,6)$ 
for 1,000 elements, and from $\calN(0,(0.2\cdot\Lambda)^{-1})\calW(\bPsi,9)$ for 2,000 elements. For each distribution
and type (1,000 or 2,000), we generated 10 datasets for each type and measured the average clustering accuracies.

We evaluated the clustering accuracy using the adjusted Rand index~\citep{HubertL85joc}. For traditional
agglomerative clustering algorithms, we assumed that we know the true number of clusters $k$ and cut the tree at corresponding level. 
For RBHC-greedy and RBHC-nnca, we selected the threshold $\lambda$ with the heuristics described in Section~\ref{subsec:sva_bhc}.  
For original BHC, we have to carefully tune the hyperparemters, and the accuracy was very sensitive to this setting. 
In the case of Poisson distribution where $\rho \sim\mathrm{Gamma}(a, b)$, we have to tune two 
hyperparameters $\{a, b\}$. For multinomial case where $\bq \sim \mathrm{Dir}(\balpha)$, we set
$\balpha = \alpha \mathbf{1}_d$ and tuned $\alpha$. 
For Gaussian case where $(\bmu,\bSigma)\sim \calN(\bm, (r\bLambda)^{-1})\calW(\nu,\bPsi^{-1})$, 
we have four hyperparameters $\{\bm, r, \nu, \bPsi\}$. We set $\bm$ to be the empirical mean of $\bX$ 
and fixed $r=0.1$ and $\nu = d + 6$. We set $\bPsi = k \bS$ where $\bS$ is the empirical covariance of $\bX$ 
and controlled $k$ according to the dimension and the size of the data. The result is summarized in Table~\ref{table:synthetic_ari}.

The accuracies of RBHC-greedy and RBHC-nnca were best for most of the cases, and the accuracies of the two methods
were almost identical expect for the multinomial distribution. BHC was best for the multinomial case
where the hyperparameter tuning was relatively easy, but showed poor performance in Poisson case (we failed
to find the best hyperparameter setting in that case). Hence, it would be a good choice to use
RBHC-greedy or RBHC-nnca which do not need careful hyperparameter tuning, and RBHC-nnca may be the best
choice considering its space and time complexity compared to RBHC-greedy.

\subsection{Experiments on Real Data}
\label{subsec:real_experiments}

We tested the agglomerative clustering algorithms on two types of real-world data. The first one
was a subset of MNIST digit database~\citep{LeCunY98procieee}. We scaled down the original $28\times 28$
to $7\times 7$ and vectorized each image to be $\bbr^{49}$ vector. Then we sampled 3,000 images
from the classes $0, 3, 7$ and $9$. We clustered this dataset with Gaussian asssumption.
The second one was visual-word data extracted from Caltech101 database~\citep{FeiFeiL2004cvprw}.
We sampled 2,033 images from "Airplane", "Mortorbikes" and "Faces-easy" classes, and extracted
SIFT features for image patches. Then we quantized those features into 1,000 visual words. We clustered
the data with multinomial assumption. Table~\ref{table:real_ari} shows the ARI values of agglomerative
clustering algorithms. As in the synthetic experiments, the accuracy RBHC-greedy and RBHC-nnca were identical,
and outperformed the traditional agglomerative clustering algorithms. BHC was best for Caltech101, where
the multinomial distribution with easy hyperparameter tuning was assumed. However, BHC was even worse
than Ward's method for MNIST case, where we failed to tune $49\times 49$ matrix $\bPsi$.

\begin{table}[ht!]
	\small
	\centering
	\caption{Average adjusted Rand index values for MNIST and Caltech101 datasets. Best ones are marked as bold face.}
	\begin{tabular}{|c|c|c|c|c|c|c|}
		\hline  & MNIST & Caltech101\\ 
		\hline single linkage & 0.000 & 0.000 \\
		complete linakge & 0.187  & 0.000  \\ 
		Ward's method & 0.485 & 0.465  \\ 
		RBHC-greedy & \textbf{0.637}  & 0.560 \\ 
		RBHC-nnca & \textbf{0.637}  & 0.560  \\ 
		BHC & 0.253 & \textbf{0.646} \\
		\hline 
	\end{tabular} 
	\label{table:real_ari}
\end{table}

\section{CONCLUSIONS}
\label{sec:conclusions}

In this paper we have presented a non-probabilistic counterpart of BHC, referred to as RBHC, using a small variance relaxation
when underlying likelihoods are assumed to be conjugate exponential families.
In contrast to the original BHC, RBHC does not requires careful tuning of hyperparameters.
We have also shown that the dissimilarity measure emerged in RBHC is
reducible with high probability, so that the nearest neighbor chain algorithm was used to
speed up the RBHC and to reduce the space complexity, leading to RBHC-nnca.
Experiments on both synthetic and real-world datasets demonstrated the validity of RBHC.

\medskip
{\bf Acknowledgements}: 
This work was supported by National Research Foundation (NRF) of Korea (NRF-2013R1A2A2A01067464)
and  the IT R\&D Program of MSIP/IITP (14-824-09-014, Machine Learning Center).

\bibliographystyle{abbrvnat}
\bibliography{sjc}

\newpage
\onecolumn
\appendix
\section{Detailed Derivation in Section \ref{subsec:reducible}}

For $\bX_c = \{\bx_i\}_{i\in c} \isim p(\bx|\beta, \btheta)$, we have
\be
p(\bar t_c|\beta,\btheta) = \exp\bigg\{ \beta|c|\ip{\bar{t}_c,\btheta} - \beta |c| \psi(\btheta)
- \sum_{i\in c} h_\beta(\bx_i)\bigg\}.
\ee
For notational simplicity, we let $\bar t_c = \by$ from now.
By the normalization property, 
\be
\beta |c|\psi(\btheta) = \log\int \exp\bigg\{  \beta|c|\ip{\by,\btheta}
- \sum_{i\in c} h_\beta(\bx_i)\bigg\}  d\by.
\ee
Differentiating both sides by $\btheta$ yields
\be
\beta|c|\frac{d\psi(\btheta)}{d\btheta} = \int \beta |c| \by\cdot p(\by|\beta,\btheta) d\by,
\quad \frac{d\psi(\btheta)}{d\btheta} = \bbe[\by].
\ee
Also, we have
\be
\beta|c|\frac{\partial^2\psi(\btheta)}{\partial\theta_j\partial\theta_k}
&=& \int \beta|c| y_j p(\by|\beta,\btheta) \bigg(
\beta|c| y_j - \beta|c|\frac{\partial\psi(\btheta)}{\partial\theta_k}\bigg) d y_j\nn
&=& \beta^2|c|^2 \bbe[y_jy_k] - \beta^2|c|^2 \bbe[\by]_j \bbe[\by]_k = 
\beta^2|c|^2 \mathrm{cov}(y_j,y_k).
\ee
Hence,
\be
\frac{1}{\beta|c|}\frac{\partial^2\psi(\btheta)}{\partial\theta_j\partial\theta_k} = \mathrm{cov}(y_j,y_k) = \int (y_j - \bbe[\by]_j)(y_k  - \bbe[\by]_k) p(\by|\beta,\btheta)dy_jdy_k.
\ee
Differentiating this again yields
\be\label{eq:third}
\frac{1}{\beta|c|}\frac{\partial^3\psi(\btheta)}{\partial\theta_j\partial\theta_k\partial\theta_l}
&=& \int (y_j-\bbe[\by]_j)(y_k-\bbe[\by]_k)(y_l-\bbe[\by]_l) p(\by|\beta,\btheta)dy_jdy_kdy_l\nn
&=& \bbe[(y_j-\bbe[\by]_j)(y_k-\bbe[\by]_k)(y_l-\bbe[\by]_l)].
\ee
Unfortunately, this relationship does not continue after the third order; the fourth derivative of 
$\psi(\btheta)$ is not exactly match to the fourth order central moment of $\by$. However,
one can easily maintain the $m$th order central moment by manipulating the $m$th order derivative
of $\psi(\btheta)$, and $m$th order central moment always have the constant term $(\beta|c|)^{-m}$.

Equation (40) of the paper is a simple consequence of the equation~\eqref{eq:third}. To prove the equation (41) of the paper, we use the following relationship:
\be
\bbe\bigg[\frac{\epsilon_\varphi(\bar{\bx}_c)}{\Delta_\varphi(\bar{\bx}_c)}\bigg]
\approx \frac{\bbe[\epsilon_\varphi(\bar{\bx}_c)]}{\bbe[\Delta_\varphi(\bar{\bx}_c)]} - \frac{\mathrm{cov}(\epsilon_\varphi(\bar{\bx}_c), \Delta_\varphi(\bar{\bx}_c))}{\bbe[\Delta_\varphi(\bar{\bx}_c)]^2} + \frac{\bbe[\epsilon_\varphi(\bar{\bx}_c)]\mathrm{var}[\Delta_\varphi(\bar{\bx}_c)]}{\bbe[\Delta_\varphi(\bar{\bx}_c)]^3}.
\ee
Now it is easy to show that this equation converges to zero when $\beta\to 0$; all the expectations
and variances can be obtained by differentiating $\psi(\btheta)$ for as many times as needed.

\end{document}